\documentclass[twoside]{article} 
\usepackage[accepted]{aistats2017}

\usepackage{hyperref}       
\usepackage{url}            
\usepackage{booktabs}       
\usepackage{amsfonts}       
\usepackage{dsfont}
\usepackage{amsmath}
\usepackage[boxed]{algorithm2e}
\usepackage{subfigure}
\usepackage{amsthm}
\usepackage{graphicx}
\usepackage{color}

\newtheorem{claim}{Claim}
\newtheorem{theorem}{Theorem}
\newtheorem{example}{Example}
\newtheorem{definition}{Definition}
\newtheorem{lemma}{Lemma}

\newcommand{\appendixContent}[1]{#1}

\newcommand{\eat}[1]{}

\begin{document}

%
\runningtitle{On the Interpretability of CPE in the Agnostic Settings}

%

\twocolumn[

\aistatstitle{On the Interpretability of Conditional Probability Estimates \\ in the Agnostic Setting}

\aistatsauthor{ Yihan Gao \And Aditya Parameswaran \And Jian Peng }

\aistatsaddress{ University of Illinois at Urbana-Champaign } ]

\begin{abstract}
We study the interpretability of conditional probability estimates for binary classification under the agnostic setting or scenario. Under the agnostic setting, conditional probability estimates do not necessarily reflect the true conditional probabilities. Instead, they have a certain calibration property: among all data points that the classifier has predicted $\mathcal{P}(Y=1|X)=p$, $p$ portion of them actually have label $Y = 1$. For cost-sensitive decision problems, this calibration property provides adequate support for us to use  Bayes Decision Theory. In this paper, we define a novel measure for the calibration property together with its empirical counterpart, and prove an uniform convergence result between them. This new measure enables us to formally justify the calibration property of conditional probability estimations, and provides new insights on the problem of estimating and calibrating conditional probabilities. 

\end{abstract}

\section{Introduction}

Many binary classification algorithms, such as naive Bayes and logistic regression, naturally produce confidence measures in the form of conditional probability of labels. These confidence measures are usually interpreted as the conditional probability of the label $y = 1$ given the feature $x$. An important research question is how to justify these conditional probabilities, i.e., how to prove the trustworthiness of such results. 

In classical statistics, this question is usually studied under the realizable assumption, which assumes that the true underlying probability distribution has the same parametric form as the model assumption. More explicitly, statisticians usually construct a parametric conditional distribution $\mathcal{P}(Y|X, \theta)$, and assume that the true conditional distribution is also of this form (with unknown $\theta$). The justification of conditional probabilities can then be achieved by using either hypothesis testing or confidence interval estimation on $\theta$. 

However, in modern data analysis workflows, the realizable assumption is often violated, e.g. data analysts usually  try out several off-the-shelf classification algorithms to identify those that work the best. This setting is often called {\em agnostic} --- essentially implying that we do not have any knowledge about the underlying distribution. Under the agnostic setting, conditional probability estimates can no longer be justified by standard statistical tools, as most hypothesis testing methods are designed to distinguish two parameter areas in the hypothesis space (e.g., $\theta < \theta_0$ v.s. $\theta \geq \theta_0$), and confidence intervals require realizable assumption to be interpretable.

In this paper, we study the interpretability of conditional probabilities in binary classification in the agnostic setting: what kind of guarantees can we have without making any assumption on the underlying distribution? Justifying these conditional probabilities is important for applications that explicitly utilize the conditional probability estimates of the labels, including medical diagnostic systems (Cooper, 1984)\eat{~\cite{cooper1984nestor}} and fraud detection (Fawcett and Provost, 1997)\eat{~\cite{fawcett1997adaptive}}. In such applications, the misclassification loss function is often asymmetric (i.e., false positive and false negative incur different loss), and accurate conditional probability estimates are crucial empirically. In particular, in medical diagnostic systems, a false positive means additional tests are needed, while a false negative could potentially be fatal. 

\noindent \textbf{Summary of Notation}

We focus on the binary classification problem in this paper. Let us first define some notations here that will be used throughout the paper:

\begin{itemize}
\item $\mathcal{X}$ denotes the discrete feature space and $\mathcal{Y} = \{\pm 1\}$ denotes the label space.
\item $\mathcal{P}$ denotes the underlying distribution over $\mathcal{X} \times \mathcal{Y}$ that governs the generation of datasets.
\item $D = \{(X_1, Y_1), \ldots, (X_n, Y_n)\}$ denotes a set of i.i.d. data points from $\mathcal{P}$.
\item A fuzzy classifier is a function from $\mathcal{X}$ to $[0,1]$ where the output denotes the estimated conditional probability of $\mathcal{P}(Y=1|X)$.
\end{itemize}

\noindent \textbf{Interpretations of Conditional Probability Estimates}

Ideally, we hope that our conditional probability estimates can be interpreted as the true conditional probabilities. This interpretation is justified if we can prove that the conditional probability estimates are close to the true values. Let $l_1(f, \mathcal{P})$ be the $l_1$ distance between the true distribution and the estimated distribution as a measure of the ``correctness'' of conditional probability estimates:
$$ l_1(f, \mathcal{P}) = \mathbb{E}_{X \sim \mathcal{P}} |f(X) - \mathcal{P}(Y=1|X)| $$
Here $X$ is a random variable representing the feature vector of a sample data point, $Y$ is the label of $X$ and $f(X)$ is a fuzzy classifier that estimates $\mathcal{P}(Y=1|X)$.
If we can prove that $l_1(f, \mathcal{P}) \leq \epsilon$ for some small $\epsilon$, then the output of $f$ can be approximately interpreted as the true conditional probability. 

Unfortunately, as we will show in this paper, it is impossible to guarantee any reasonably small upper bound for $l_1(f, \mathcal{P})$ under the agnostic assumption. In fact, as we will demonstrate in this paper, for the cases where we have to make the agnostic assumption, the estimated conditional probabilities are usually no longer close to the true values in practice.

Therefore, instead of trying to bound the $l_1$ distance, we develop an alternative interpretation for these conditional probability estimates. We introduce the following calibration definition for fuzzy classifiers:
\begin{definition}
	Let $\mathcal{X}$ be the feature space, $\mathcal{Y} = \{\pm 1\}$ be the label space and $\mathcal{P}$ be the distribution over $\mathcal{X} \times \mathcal{Y}$. Let $f: \mathcal{X} \rightarrow [0, 1]$ be a fuzzy classifier, then we say $f$ is calibrated if for any $p_1 < p_2$, we have:
	$$ \mathbb{E}_{X \sim \mathcal{P}} [\mathds{1}_{p_1 < f(X) \leq p_2} f(X)] = \mathcal{P}(p_1 < f(X) \leq p_2, Y = 1)$$
\end{definition}

Intuitively, a fuzzy classifier is calibrated if its output correctly reflects the relative frequency of labels among instances they believe to be similar. For instance, suppose the classifier output $f(X) = p$ for $n$ data points, then roughly there are $np$ data points with label $Y=1$. We also define a measure of how close $f$ is to be calibrated:
\begin{definition}\label{dfn_measure}
	A fuzzy classifier $f$ is $\epsilon$-calibrated if
	\begin{align*}
	c(f) = & \sup_{p_1 < p_2} | \mathcal{P}(p_1 < f(X) \leq p_2, Y = 1) \\ & - \mathbb{E}_{X \sim \mathcal{P}} [\mathds{1}_{p_1 < f(X) \leq p_2} f(X)]| \leq \epsilon
	\end{align*}
	
	$f$ is $\epsilon$-empirically calibrated with respect to $D$ if
	\begin{align*}
	c_{emp}(f, D) = & \frac{1}{n} \sup_{p_1 < p_2} | \sum_{i=1}^n \mathds{1}_{p_1 < f(X_i) \leq p_2, Y_i = 1} \\ & - \sum_{i=1}^n \mathds{1}_{p_1 < f(X_i) \leq p_2} f(X_i)]| \leq \epsilon
	\end{align*}
	where $D = \{(X_i, Y_i), \ldots, (X_n, Y_n)\}$ is a size $n$ dataset consisting of i.i.d. samples from $\mathcal{P}$.
\end{definition}

Note that the empirical calibration measure $c_{emp}(f, D)$ can be efficiently computed on a finite dataset. We further prove that under certain conditions, $c_{emp}(f, D)$ converges uniformly to $c(f)$ over all functions $f$ in a hypothesis class. Therefore, the calibration property of these classifiers can be demonstrated by showing that they are empirically calibrated on the training data.

The calibration definition is motivated by analyzing the properties of commonly used conditional probability estimation algorithms: many such algorithms will generate classifiers that are naturally calibrated. 
Our calibration property justifies the common practice of using calibrated conditional probability estimates as true conditional probabilities: we show that if the fuzzy classifier is calibrated and the output of the classifier is the only source of information, then the optimal strategy is to apply Bayes Decision Rule on the conditional probability estimates. 

The uniform convergence result of $c_{emp}(f, D)$ and $c(f)$ has several applications. First, it can be directly used to prove a fuzzy classifier is (almost) calibrated, which is necessary for the conditional probability estimates to be interpretable. Second, it suggests that we need to minimize the empirical calibration measure to obtain calibrated classifiers, which is a new direction for designing conditional probability estimation algorithms. Finally, taking an uncalibrated conditional probability estimates as input, we can calibrate them by minimizing the calibration measure. In fact, one of the most well-known calibration algorithm, the isotonic regression algorithm, can be interpreted this way.

\noindent \textbf{Paper Outline}

The rest of this paper is organized as following. In Section~\ref{sec_motivation}, we argue that the $l_1$ distance cannot be provably bounded under the agnostic assumption (Theorem~\ref{thm_nobound}) and then motivate our calibration definition. In Section~\ref{sec_main} we present the uniform convergence result (Theorem~\ref{thm_main}) and discuss the potential applications. In Section~\ref{sec_experiment}, we conduct experiments to illustrate the behavior of our calibration measure on several common classification algorithms. 
 
\noindent \textbf{Related Work}


Our definition of calibration is similar to the definition of calibration in prediction theory (Foster and Vohra, 1998)\eat{~\cite{foster1998asymptotic}}, where the goal is also to make predicted probability values match the relative frequency of correct predictions. In prediction theory, the problem is formulated from a game-theoretic point of view: the sequence generator is assumed to be malevolent, and the goal is to design algorithms to achieve this calibration guarantee no matter what strategy the sequence generator uses.

To the best of our knowledge, there is no other work addressing the interpretability of conditional probability estimates in agnostic cases. Our definition of calibration is also connected to the problem of calibrating conditional probability estimates, which has been studied in many papers (Zadrozny and Elkan, 2002) (Platt, 1999)\eat{~\cite{zadrozny2001obtaining, platt1999probabilistic}}. 

\section{The Calibration Definition: Motivation \& Impossibility Result}\label{sec_motivation}

\subsection{Impossibility result for $l_1$ distance}

Recall that the $l_1$ distance between $f$ and $\mathcal{P}$ is defined as:
$$ l_1(f, \mathcal{P}) = \mathbb{E}_{X \sim \mathcal{P}}|f(X) - \mathcal{P}(Y=1|X)| $$
Suppose $f$ is our conditional probability estimator that we learned from the training dataset. We attempt to prove that the $l_1$ distance between $f$ and $\mathcal{P}$ is small under the agnostic setting. With the agnostic setting, we do not know anything about $\mathcal{P}$, and the only tool we can utilize is a validation dataset $D_{val}$ that consists of i.i.d. samples from $\mathcal{P}$. Therefore, our best hope would be a prover $A_f(D)$ that:
\begin{itemize}
	\item Returns $1$ with high probability if $l_1(f, \mathcal{P})$ is small.
	\item Returns $0$ with high probability if $l_1(f, \mathcal{P})$ is large.
\end{itemize}

The following theorem states that no such prover exists, and the proof can be found in the appendix.

\begin{theorem}\label{thm_nobound}

Let $\mathcal{Q}$ be a probability distribution over $\mathcal{X}$, and $f: \mathcal{X} \rightarrow [0,1]$ be a fuzzy classifier. Define $B_f$ as:
$$ B_f = \mathbb{E}_{X \sim \mathcal{Q}} \min(f(X), 1 - f(X)) $$
If we have that $ \forall x \in \mathcal{X}, \mathcal{Q}(x) < \frac{1}{10n^2} $, 
then there is no prover $A_f : \{\mathcal{X} \times \mathcal{Y}\}^n \rightarrow \{0, 1\}$ for $f$ satisfying the following two conditions:

For any $\mathcal{P}$ over $\mathcal{X} \times \mathcal{Y}$ such that $\mathcal{P}_X = \mathcal{Q}$ (i.e., $\forall x \in \mathcal{X}, \sum_{y \in \mathcal{Y}} \mathcal{P}(x, y) = \mathcal{Q}(x)$), suppose $D_{val} \in \{\mathcal{X} \times \mathcal{Y}\}^n$ is a validation dataset consisting of $n$ i.i.d. samples from $\mathcal{P}$:
\begin{enumerate}
\item
If $l_1(f, \mathcal{P}) = 0$, then $\mathbf{P}_{D_{val}}(A_f(D_{val}) = 1) > \frac{2}{3}$.
\item
If $l_1(f, \mathcal{P}) > \frac{B_f}{2}$, then $\mathbf{P}_{D_{val}}(A_f(D_{val}) = 1) < \frac{1}{3}$.
\end{enumerate}
\end{theorem}

We made the assumption in Theorem~\ref{thm_nobound} to exclude the scenario where a significant amount of probability mass concentrates on a few data points so that their corresponding conditional probability can be estimated via repeated sampling. Note that the statement is not true in the extreme case where all probability mass concentrates on one single data point (i.e., $\exists x \in X, Q(x) = 1$). The assumption is true when the feature space $\mathcal{X}$ is large enough such that it is almost impossible for any data point to have significant enough probability mass to get sampled more than once in the training dataset.

The significance of Theorem~\ref{thm_nobound} is that any attempt to guarantee a small upper bound of $l_1(f, \mathcal{P})$ would definitely fail. Thus, we can no longer interpret the conditional probability estimates as the true conditional probabilities under the agnostic setting. This result motivates us to develop a new measure of ``correctness'' to justify the conditional probability estimates.

\subsection{$l_1(f, \mathcal{P})$ in practice}

The fact that we cannot guarantee an upper bound of the $l_1$ distance is not merely a theoretical artifact. In fact, in the cases where we need to make the agnostic assumption, the value of $l_1(f, \mathcal{P})$ is often very large in practice. Here we use the following document categorization example to demonstrate this point.

\begin{example}\label{example_logistic}
Denote $Z$ to be the collection of all English words. In this problem the feature space $\mathcal{X} = Z^*$ is the collection of all possible word sequences, and $\mathcal{Y}$ denotes whether this document belongs to a certain topic (say, football). Denote $\mathcal{P}$ as the following data generation process: $X$ is generate from the Latent Dirichlet Allocation model (Blei et al., 2003), and $Y$ is chosen randomly according to the topic mixture.

We use logistic regression, which is parameterized by a weight function $w:Z \rightarrow \mathbb{R}$, and two additional parameters $a$ and $b$. For each document $X = z_1 z_2 \ldots z_k$, the output of the classifier is:
$$ f(X) = \frac{1}{1 + \exp(- a \sum_{i=1}^k w(z_i) - b)} $$
\end{example}


The reason that we are using automatically generated documents instead of true documents here is that the conditional probabilities $P(Y|X)$ are directly computable (otherwise we cannot evaluate $l_1(f, \mathcal{P})$ and other measures). We conducted an experimental simulation for this example, and the experimental details can be found in the appendix. Here we summarize the major findings: the logistic regression classifier has very large $l_1$ error, which is probably due to the discrepancy between the logistic regression model and the underlying model. However, the logistic regression classifier is almost naturally calibrated in this example. This is not a coincidence, and we will discuss the corresponding intuition in Section~\ref{sec_cal_motivation}.

\subsection{The Motivation of the Calibration Measure}\label{sec_cal_motivation}

Let us revisit Example~\ref{example_logistic}. This time, we fix the word weight function $w$. In this case, every document $X$ can be represented using a single parameter $w(X) = \sum_i w(z_i)$, and we search for the optimal $a$ and $b$ such that the log-likelihood is maximized. 
This is illustrated in Figure~\ref{fig_example}.

\begin{figure}[h]
	\begin{center}
		\includegraphics[height = 3cm]{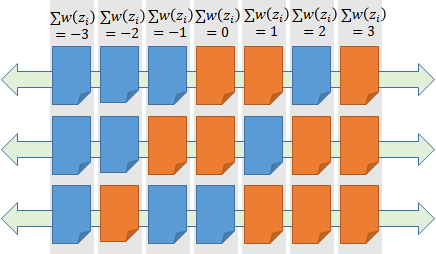}
		\caption{Illustration of Example~\ref{example_logistic}}\label{fig_example}
	\end{center}
\end{figure}

Now, intuitively, to maximize the log-likelihood, we need the sigmoid function $\frac{1}{1 + \exp( - a w(X) - b)}$ to match the conditional probability of $Y$ conditioned on $w(X)$: $\mathcal{P}(Y=1|w(X))$. Therefore, for the optimal $a$ and $b$, we could say that the following property is roughly correct:
$$ \mathcal{P}(Y=1|w(X)) \approx \frac{1}{1 + \exp(- a w(X) - b)} $$
In other words, 
$$ \forall 0 \leq p \leq 1, \mathbb{E}[\mathcal{P}(Y=1|X)|f(X) = p] \approx p $$

Let us examine this example more closely. The reason why the logistic regression classifier tells us that $f(X) \approx p$ is because of the following: among all the documents with similar weight $w(X)$, about $p$ portion of them actually belong to the topic in the training dataset. This leads to an important observation: logistic regression estimates the conditional probabilities by computing the relative frequency of labels among documents it believes to be similar.

This behavior is not unique to logistic regression. Many other algorithms, including decision tree classifiers, nearest neighbor (NN) classifiers, and neural networks, exhibit similar behavior:
\begin{itemize}
	\item
	In decision trees, all data points reaching the same decision leaf are considered similar.
	\item
	In NN classifiers, all data points with the same nearest neighbors are considered similar.
	\item
	In neural networks, all data points with the same output layer node values are considered similar.
\end{itemize}

We can abstract the above conditional probability estimators as the following two-step process:
\begin{enumerate}
\item
Partition the feature space $\mathcal{X}$ into several regions.
\item
Estimate the relative frequency of labels among all data points inside each region.
\end{enumerate}


The definition of the calibration property follows easily from the above two-step process. We can argue that the classifier is approximately calibrated, if for each region $S$ in the feature space $\mathcal{X}$, the output conditional probability of data points in $S$ is close to the actual relative frequency of labels in $S$. The definition for the calibration property then follows from the fact that all data points inside each region have the same output conditional probabilities.
\begin{align*}
\forall p_1 < p_2, \quad & \mathcal{P}(p_1 < f(X) \leq p_2, Y = 1) \\ & = \mathbb{E}_{X \sim \mathcal{P}} [\mathds{1}_{p_1 < f(X) \leq p_2} f(X)]
\end{align*}

\noindent \textbf{Using Calibrated Conditional Probabilities in Decision Making}

The calibration property justifies the common practice of using estimated conditional probabilities in decision making. Consider the binary classification problem with assymetric misclassification loss: we lose $a$ points for every false positive and $b$ points for every false negative. In this case, the best decision strategy is to predict $1$ if $\mathcal{P}(Y=1|X) \geq \frac{a}{a + b}$ and predict $-1$ otherwise. Now consider the case when we do not know $\mathcal{P}(Y=1|X)$, but only know the value of $f(X)$ instead. If we can only use $f(X)$ to make decision, and $f$ is calibrated, then the best strategy is to use $f(X)$ in the same way as $\mathcal{P}(Y=1|X)$ (the proof can be found in the appendix):
\begin{claim}\label{clm_bayes_decision}
Suppose we are given a calibrated fuzzy classifier $f: \mathcal{X} \rightarrow [0,1]$, we need to make decisions solely based on the output of $f$. Denote our decision as $g:[0, 1] \rightarrow \{\pm 1\}$ (i.e., our decision for $X$ is $g(f(X))$). Then the optimal strategy $g^*$ to minimize the expected loss is the following:
\begin{align*}
g^*(x) = \left\{ \begin{array}{lr} 1 & x \geq \frac{a}{a + b} \\ -1 & x < \frac{a}{a + b} \end{array} \right.
\end{align*}
\end{claim}

\section{Uniform Convergence of the Calibration Measure}\label{sec_main}



\subsection{The Uniform Convergence Result}\label{sec_thm}

Let $\mathcal{G}$ be a collection of functions from $\mathcal{X} \times \mathcal{Y}$ to $[0, 1]$, the Rademacher Complexity (Bartlett and Mendelson, 2003)\eat{~\cite{bartlett2003rademacher}}
\footnote{Our definition of Rademacher Complexity comes from Shalev-Shwartz and Ben-David's textbook (2014)\eat{~\cite{shalev2014understanding}}, which is slightly different from the original definition in Bartlett and Mendelson's paper (2003)\eat{~\cite{bartlett2003rademacher}}.}
 of $\mathcal{G}$ with respect to $D$ is defined as (Shalev-Shwartz and Ben-David, 2014)\eat{~\cite{shalev2014understanding}}:
$$ R_{D}(\mathcal{G}) = \frac{1}{n} \mathbb{E}_{\sigma \sim \{\pm 1\}^n} [ \sup_{g \in \mathcal{G}} \sum_{i=1}^n \sigma_i g(x_i, y_i) ] $$

Then we have the following result:

\begin{theorem}\label{thm_main}
	Let $\mathcal{F}$ be a set of fuzzy classifiers, i.e., functions from $\mathcal{X}$ to $[0, 1]$. Let $\mathcal{H}$ be the set of binary classifiers obtained by thresholding the output of fuzzy classifiers in $\mathcal{F}$:
	$$\mathcal{H} = \{\mathds{1}_{p_1 < f(x) \leq p_2} : p_1, p_2 \in \mathbb{R}, f \in \mathcal{F}\}$$
	Suppose the Rademacher Complexity of $\mathcal{H}$ satisfies:
	$$ \mathbb{E}_{D} R_{D}(\mathcal{H}) + \sqrt{\frac{2 \ln(8 / \delta)}{n}} < \frac{\epsilon}{2}$$
	Then,
	$$ \mathbf{Pr}_D( \sup_{f \in \mathcal{F}} |c(f) - c_{\text{emp}}(f, D)| > \epsilon ) < \delta $$
\end{theorem}

The proof of this theorem, together with a discussion on the hypothesis class $\mathcal{H}$, can be found in the appendix.

\subsection{Applications of Theorem~\ref{thm_main}}

\subsubsection{Verifying the calibration of classifier}

The first application of Theorem~\ref{thm_main} is that we can verify whether the learned classifier $f$ is calibrated. For simple hypothesis spaces $\mathcal{F}$ (e.g., logistic regression), the corresponding hypothesis space $\mathcal{H}$ has low Rademacher Complexity. In this case, Theorem~\ref{thm_main} naturally guarantees the generalization of calibration measure. 

There are also cases where the Rademacher Complexity of $\mathcal{H}$ is not small. One notable example is SVM classifiers with Platt Scaling (Platt, 1999)\eat{~\cite{platt1999probabilistic}}:
\begin{claim}\label{clm_svm}
	Let $\mathcal{X} \subseteq \mathbb{R}^d$ and $\forall x \in \mathcal{X}, ||x||_2 \leq 1$. Let $\mathcal{F}$ be the following hypothesis class:
	\begin{align*}
	\mathcal{F} = \{x \rightarrow \frac{1}{1 + \exp(a w^T x + b)} : \\ w \in \mathbb{R}^d, ||w||_2 \leq B, a, b \in \mathbb{R} \}
	\end{align*}
	If the training data size $n < d$ and the training data $X_i$ are linearly independent, then $R_D(\mathcal{H}) = \frac{1}{2}$.
\end{claim}
The proof can be found in the appendix. In the case of SVM, the dimensionality of the feature space is usually much larger than the training dataset size (this is especially true for kernel-SVM). In this situation, we can no longer verify the calibration property using only the training data, and we have to keep a separate validation dataset to calibrate the classifier (as suggested by Platt (1999))\eat{\cite{platt1999probabilistic}}. When verifying the calibration of classifier on a validation dataset. The hypothesis class $\mathcal{F} = \{f\}$, and it is easy to verify that $\mathbb{E}_D R_D(\mathcal{H})$ is $O(\sqrt{\log n/n})$ in this case. Therefore, with enough validation data, we can still bound the calibration measure.

\subsubsection{Implications on Learning Algorithm Design}

Standard conditional probability estimation usually maximizes the likelihood to find the best classifier within the hypothesis space. However, since we can only guarantee the conditional probability estimates to be calibrated under the agnostic assumption, any calibrated classifier is essentially as good as the maximum likelihood estimation in terms of interpretability. Therefore, likelihood maximization is not necessarily the only method for estimating conditional probabilities.

There are other loss functions that are already widely used for binary classification. For example, hinge loss is at the foundation of large margin classifiers. Based on our discussion in this paper, we believe that these loss functions can also be used for conditional probability estimation. For example, Theorem~\ref{thm_main} suggests the following constrained optimization problem:
$$ \min \mathcal{L}(f, D) \quad s.t. \quad c_{emp}(f, D) = 0 $$
where $\mathcal{L}(f, D)$ is the loss function we want to minimize. By optimizing over the space of empirically calibrated classifiers, we can ensure that the resulting classifier is also calibrated with respect to $\mathcal{P}$. 

In fact, the conditional probability estimation algorithm developed by Kakade et al. (2011) already implicitly follows this framework (more elaboration on this point can be found in the appendix). 
We believe that many more interesting algorithms can be developed along this direction.

\subsubsection{Connection to the Calibration Problem}\label{sec_calibration_alg}

Suppose that we are given an uncalibrated fuzzy classifier $f_0: \mathcal{X} \rightarrow [0,1]$, and we want to find a function $g$ from $[0,1]$ to $[0,1]$, so that $g \circ f_0$ presents a better conditional probability estimation. This is the problem of classifier calibration, which has been studied in many papers (Zadrozny and Elkan, 2002) (Platt, 1999)\eat{~\cite{zadrozny2001obtaining, platt1999probabilistic}}.

Traditionally, calibration algorithms find the best link function $g$ by maximizing likelihood or minimizing squared loss. In this paper, we suggest a different approach to the calibration problem. We can find the best $g$ by minimizing the empirical calibration measure $c_{\text{emp}}(g \circ f_0)$. Let us assume w.l.o.g. that the training dataset $D = \{(x_1, y_1), \ldots, (x_n, y_n)\}$ satisfies 
$$g(f_0(x_1)) \leq \ldots \leq g(f_0(x_n))$$
Then we have, 
\begin{align*}
& c_{\text{emp}}(g \circ f_0, D) \\ = & \frac{1}{n} \sup_{p_1, p_2} | \sum_{i=1}^n \mathds{1}_{p_1 < g(f_0(x_i)) \leq p_2} (\mathds{1}_{y_i=1} - g(f_0(x_i))) | \\
\leq & \frac{1}{n} \max_{a,b} | \sum_{a<i\leq b} (\mathds{1}_{y_i=1} - g(f_0(x_i))) |
\end{align*}

This expression can be used as the objective function for calibration: we search over the space of hypothesis $\mathcal{G}$ to find a function $g$ that minimizes this objective function. Compared to other loss functions, the benefits of minimizing this objective function is that the resulting classifier is more likely to be calibrated, and therefore provides more interpretable conditional probability estimates. 

In fact, one of the most well-known calibration algorithms, the isotonic regression algorithm, can be viewed as minimizing this objective function:

\begin{claim}\label{clm_PAV_optimality}
	Let $\mathcal{G}$ be the set of all continuous nondecreasing functions from $[0,1]$ to $[0,1]$. Then the optimal solution that minimizes the squared loss $\min_{g \in \mathcal{G}} \sum_{i=1}^n (\mathds{1}_{y_i=1} - g(f_0(x_i)))^2$ also minimizes $\min_{g \in \mathcal{G}} \max_{a,b} | \sum_{a<i\leq b} (\mathds{1}_{y_i=1} - g(f_0(x_i))) |$
\end{claim}

The proof can be found in the appendix. Using this connection we proved several interesting properties of the isotonic regression algorithm, which can also be found in the appendix.
\section{Empirical behavior of the calibration measure}\label{sec_experiment}

In this section, we conduct some preliminary experiments to demonstrate the behavior of the calibration measure on some common algorithms. We use two binary classification datasets from the UCI Repository\footnote{These datasets are chosen from the datasets used in Niculescu-Mizil and Caruana's work (2005). We only used two datasets because the experiments are only explorative (i.e., identifying potential properties of the calibration measure). More rigorous experiments are needed to formally verify these properties.}: 
ADULT\footnote{https://archive.ics.uci.edu/ml/datasets/Adult} and COVTYPE\footnote{https://archive.ics.uci.edu/ml/datasets/Covertype}. COVTYPE has been converted to a binary classification problem by treating the largest class as positive and the rest as negative. Five algorithms have been used in these experiments: naive Bayes(NB), boosted decision trees, SVM\footnote{For SVM and boosting, we rescale the output score to $[0,1]$ by $(x - \min) / (\max - \min)$ as in Niculescu-Mizil and Caruana's paper (2005)\eat{~\cite{niculescu2005predicting}}}
, logistic regression(LR), random forest(RF). 

\begin{figure}[h]
	\centering
	\includegraphics[width = 8cm, height = 3cm]{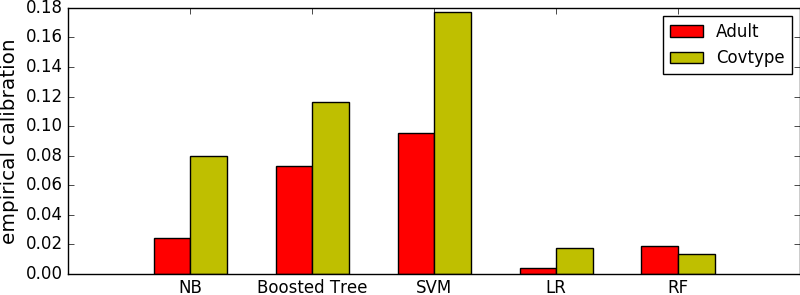}
	\caption{The empirical calibration error}\label{fig_cemp}
\end{figure}

Figure~\ref{fig_cemp} shows the empirical calibration error $c_{emp}$ on test datasets for all methods. From the experimental results, it appears that Logistic Regression and Random Forest naturally produce calibrated classifiers, which is intuitive as we discussed in the paper. The calibration measure of Naive Bayes seems to be depending on the dataset. For large margin methods (SVM and boosted trees), the calibration measures are high, meaning that they are not calibrated (on these two datasets). 

There is also an interesting connection between the calibration error and the benefit of applying a calibration algorithm, which is illustrated in Figure~\ref{fig_ratio}. In this experiment, we used a loss parameter $p$ to control the asymmetric loss: each false negative incurs $1 - p$ cost and each false positive incurs $p$ cost. All the algorithms are first trained on the training dataset, then calibrated on a separate validation set of size $2000$ using isotonic regression. For each algorithm, we compute the prior-calibration and post-calibration average losses on the testing dataset using the following decision rule: For each data point $X$, we predict $Y=1$ if and only if we predict that $\mathbf{Pr}(Y=1|X) \geq p$. Finally, we report the ratio between two losses:
$$ \text{loss ratio} = \frac{\text{the average loss after calibration}}{\text{the average loss before calibration}} $$
\begin{figure}[h]
	\centering
	\subfigure[Adult]{
		\includegraphics[width = 5cm, height = 3cm]{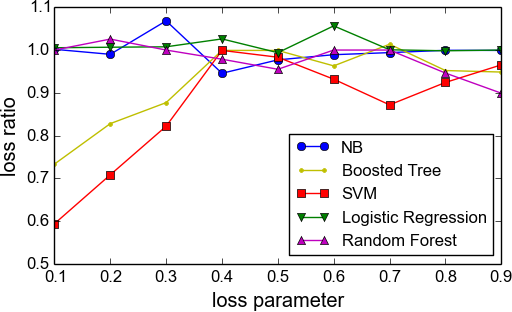}
		\label{fig_adult}
	}
	\subfigure[Covtype]{
		\includegraphics[width = 5cm, height = 3cm]{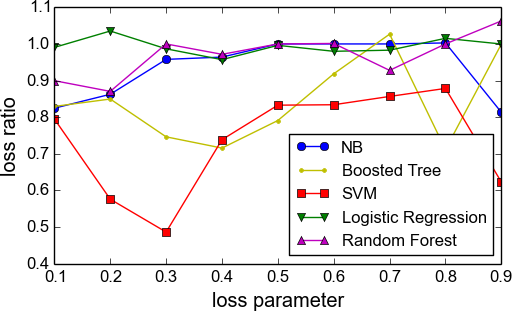}
		\label{fig_covtype}
	}
	\caption{The loss ratio on two datasets}\label{fig_ratio}
\end{figure}
As we can see in the Figure~\ref{fig_ratio}, the calibration procedure on average reduces the cost by 3\%-5\% for naive Bayes and random forest, 20\% for SVM, 12\% for boosted trees, and close to 0\% for logistic regression. Comparing with the results in Figure~\ref{fig_cemp}, the two algorithms that benefit most from calibration (i.e., SVM and boosted trees) also has high empirical calibration error. This result suggests that if an algorithm already has a low calibration error to begin with, then it is not likely to benefit much from the calibration process. This finding could potentially help us decide whether we need to calibrate the current classifier using isotonic regression (Niculescu-Mizil and Caruana, 2005).
\section{Conclusion}\label{sec_conclusion}

In this paper, we discussed the interpretability of conditional probability estimates under the agnostic assumption. We proved that it is impossible to upper bound the $l_1$ error of conditional probability estimates under such scenario. Instead, we defined a novel measure of calibration to provide interpretability for conditional probability estimates. The uniform convergence result between the measure and its empirical counterpart allows us to empirically verify the calibration property without making any assumption on the underlying distribution: the classifier is (almost) calibrated if and only if the empirical calibration measure is low. Our result provides new insights on conditional probability estimation: ensuring empirical calibration is already sufficient for providing interpretable conditional probability estimates, and thus many other loss functions (e.g., hinge loss) can also be utilized for estimating conditional probabilities. 


\section*{Acknowledgement}

We thank the anonymous reviewers for their valuable feedback. We acknowledge support from grant IIS-1513407 and IIS-1633755 awarded by the NSF, grant 1U54GM114838 awarded by NIGMS and 3U54EB020406-02S1 awarded by NIBIB through funds provided by the trans-NIH Big Data to Knowledge (BD2K) initiative (www.bd2k.nih.gov), and funds from  Adobe, Google, the Sloan Foundation, and the Siebel Energy Institute. The content is solely the responsibility of the authors and does not necessarily represent the official views of the funding agencies and organizations.

\bibliographystyle{abbrv}
\bibliography{paper}

\nocite{*}
\appendixContent{
	\newpage
	\textit{ }
	\newpage
	\appendix
	\section*{Appendix}

\subsection*{Experimental Simulation of Example~\ref{example_logistic}}

Here we experimentally simulate Example~\ref{example_logistic} to illustrate that logistic regression classifier has large $l_1$ error. We use Latent Dirichlet Allocation (LDA) (Blei et al., 2003), the state of the art generative model for documents, to generate datasets. The detailed experiment settings are listed below:

\begin{itemize}
	\item
	The dataset consists of $20000$ documents, the number of topics is $20$, the dictionary size is $1000$, and the average number of words in each document is $200$.
	\item
	We use the non-informative Dirichlet prior $\alpha = (1,1,\ldots,1)$ over topics. The word distribution in each topic follows power law with a random order among words.
	\item
	For each document, we randomly sample with replacement $10$ topic labels from the topic distribution.
\end{itemize}

Table~\ref{tbl_logistic} reports the mean experiment results and the standard deviation across five runs. For reference we also include the relative frequency of labels, and the $l_1$ error achieved by the trivial classifier that always output the global relative frequency of labels as conditional probability.

\begin{table}[h]
	\centering
	\small
	\begin{tabular}{|c|c|}
		\hline
		Average $l_1$ Error & Empirical Calibration \\ 
		\hline
		$0.1270 \pm 0.0008$ & $0.0083 \pm 0.0003$ \\ 
		\hline
		Trivial $l_1$ Error & Frequency of Labels\\
		\hline
		$0.2022 \pm 0.0001$ & $0.3448 \pm 0.0001$ \\
		\hline
	\end{tabular}
	\caption{$L_1$ error and empirical calibration}\label{tbl_logistic}
\end{table}

As we can see from Table~\ref{tbl_logistic}, the logistic regression only achieves $0.13$ average $l_1$ error, while even the trivial classifier can achieve $0.2$. This implies that logistic regression performed very badly in this example. However, as we can see from Table~\ref{tbl_logistic}, the empirical calibration measure of logistic regression classifier is relatively low ($0.01$), indicating that the classifier is almost calibrated. 

\subsection*{Proof of Theorem~\ref{thm_nobound}}

\begin{proof}

The proof relies on the following lemma:
\begin{lemma}
Let $\mathcal{P}$ be a distribution over $\mathcal{X} \times \mathcal{Y}$. Let $D$ be a size $n$ i.i.d. sample set from $\mathcal{P}$. Let $V$ be a verifier of $\mathcal{P}$ given $D$ (i.e., $V$ is a function from $\{\mathcal{X} \times \mathcal{Y}\}^n$ to $\{0, 1\}$), such that
\begin{enumerate}
\item
With probability at least $1 - \delta_1$, a dataset $D$ with $n$ i.i.d. samples from $\mathcal{P}$ will pass $V$:
$$ \mathbf{Pr}_D(V(D) = 1) \geq 1 - \delta_1 $$
\item
With probability at least $1 - \delta_2$, a dataset $D$ with $n$ i.i.d. samples from $\mathcal{P}$ satisfies:
$$ \mathbf{Pr}(\forall i \not= j, X_i \not= X_j) \geq 1 - \delta_2 $$
\end{enumerate}

Then there exists another probability distribution $\mathcal{P}'$ such that:
\begin{enumerate}
\item
With probability at least $1 - \delta_1 - \delta_2$, a data $D'$ with $n$ i.i.d. samples from $\mathcal{P}'$ will also pass $V$.
$$ \mathbf{Pr}_{D'}(V(D') = 1) \geq 1 - \delta_1 - \delta_2 $$
\item
$$ \forall X \in \mathcal{X}, \sum_{Y \in \mathcal{Y}} \mathcal{P}(X, Y) =  \sum_{Y \in \mathcal{Y}} \mathcal{P}'(X, Y) $$
\item
$$ \forall X \in \mathcal{X}, \mathcal{P}'(Y=1|X) = 0 \textit{  or  } 1 $$
\end{enumerate}
\end{lemma}
\begin{proof}
First we construct the following distribution over all possible $\mathcal{P}'$ satisfying the last two conditions:
$$ \mathbf{Pr}(\mathcal{P}') = \prod_{X \in \mathcal{X}} Q(P'(Y=1|X), P(Y=1|X)) $$
where $Q(p', p)$ is defined as:
\begin{align*}
Q(p', p) = \left\{ \begin{array}{lr} p & p'
 = 1 \\ 1 - p & p' = 0 \end{array} \right.
\end{align*}

Now it is sufficient to show that if we sample $\mathcal{P}'$ according to the above distribution and then sample $D'$ from $\mathcal{P}'$, then with probability at least $1 - \delta_1 - \delta_2$, $D'$ will pass $V$. Assuming this is true, then at least one distribution $\mathcal{P}'$ have to satisfy the first condition, and thereby proved the existence of $\mathcal{P}'$.

To compute the probability that $D'$ would pass $V$, denote $D_X = \{X_1, X_2, \ldots, X_n\}$ and $D_Y = \{Y_1, Y_2, \ldots, Y_n\}$. Note that all $\mathcal{P}'$ has the same marginal distribution over $\mathcal{X}$, therefore:
\begin{align*} & \mathbf{Pr}_{\mathcal{P}', D'}(V(D') = 1) = \sum_{\mathcal{P}'} \mathbf{Pr}(\mathcal{P}') \sum_{D'} \mathbf{Pr}(D'|\mathcal{P}') V(D') \\
= & \sum_{D_X'} \mathbf{Pr}(D_X') \sum_{\mathcal{P}'} \mathbf{Pr}(\mathcal{P}') \sum_{D_Y'} \mathbf{Pr}(D_Y'|\mathcal{P}', D_X') V(D')
\end{align*}

We only consider all those $D_X'$ with distinct $X_i$ values. Based on the assumption, such $D_X'$ accounts for at least $1 - \delta_2$ of the probability mass. Now the important observation is that for every fixed $D_X'$ with distinct $X$ values, the marginal distribution of $D_Y'$ given $D_X'$ (i.e. marginalize over $\mathcal{P}'$) is exactly $\mathcal{P}(D_Y'|D_X')$, the distribution that we sample labels independently from $\mathcal{P}(Y|X)$ for each $X_i'$ in $D_X'$:
\begin{align*} & \sum_{D_X'} \mathbf{Pr}(D_X') \sum_{\mathcal{P}'} \mathbf{Pr}(\mathcal{P}') \sum_{D_Y'} \mathbf{Pr}(D_Y'|\mathcal{P}', D_X') V(D') \\
\geq & \sum_{D_X'} \mathbf{Pr}(D_X')  \mathds{1}_{\forall i \not= j, X_i' \not= X_j'} \sum_{D_Y'} \mathbf{Pr}(D_Y'|\mathcal{P}, D_X') V(D') 
\end{align*}
The latter probability is actually the probability that $D'$ will pass $V$ and have distinct $X$ values at the same time. Based on the assumptions in the lemma, it occurs with probability at least $1 - \delta_1 - \delta_2$. 
\end{proof}

Now given this lemma, the proof of Theorem~\ref{thm_nobound} is easy: We show that if any prover $A_f$ satisfies the two conditions in the theorem, it can be used as the verifier $V$ in the lemma such that no $\mathcal{P}'$ can satisfy all three conditions. 

Let $\delta_1 = \frac{1}{3}$, then the first assumption in the lemma is satisfied, also since $\forall x \in \mathcal{X}, \mathcal{Q}(x) < \frac{1}{10 n^2}$, we have:
$$ \forall i \not= j, \mathbf{Pr}(X_i = X_j) = \sum_{x} \mathcal{Q}(x)^2 \leq \frac{1}{10n^2} $$

By a union bound, we have:
$$ \mathbf{Pr}(\forall i \not=j, X_i \not= X_j) \geq \frac{9}{10} $$

Therefore we can set $\delta_2 = 0.1$. By the above lemma, there exists another $\mathcal{P}'$ such that
$$ \mathbf{Pr}_{D' \sim \mathcal{P}'}(A_f(D')) \geq 1 - \frac{1}{3} - \frac{1}{10} $$
and
$$ \forall X \in \mathcal{X}, Y \in \mathcal{Y}, \mathcal{P}'(X, Y) = 0 \textit{  or  } 1 $$
On the other hand, note that the $l_1$ distance between $\mathcal{P}'$ and $\mathcal{P}$ is at least $B$, then by the properties of $A_f$, $D'$ cannot pass $A_f$ with probability greater than $\frac{1}{3}$. This contradicts our earlier result. Therefore no such $A_f$ can exist.

\end{proof}

\subsection*{Proof of Claim~\ref{clm_bayes_decision}}
\begin{proof}
The expected loss is
$$ a \mathcal{P}(g(f(X)) = 1, Y = -1) + b \mathcal{P}(g(f(X)) = -1, Y = 1) $$
Define $ S = \{f(X) : X \in \mathcal{X}\} $, then we have:
\begin{align*}
& a \mathcal{P}(g(f(X)) = 1, Y = -1) + b \mathcal{P}(g(f(X)) = -1, Y = 1) \\
= & \sum_{p \in S} \sum_{X:f(X) = p} [ a \mathds{1}_{g(p) = 1} \mathcal{P}(Y = -1, X) + \\
  & b \mathds{1}_{g(p) = -1} \mathcal{P}(Y = 1, X) ] \\
= & \sum_{p \in S} [ a \mathds{1}_{g(p) = 1} \sum_{X:f(X) = p} \mathcal{P}(Y = -1, X) + \\
  & b \mathds{1}_{g(p) = -1} \sum_{X:f(X) = p} \mathcal{P}(Y = 1, X)]
\end{align*}
Therefore, the optimal $g^*$ has $g^*(p)=1$ if and only if:
$$ a \sum_{X:f(X) = p} \mathcal{P}(Y = -1, X) \leq b \sum_{X:f(X) = p} \mathcal{P}(Y = 1, X) $$
Which is equivalent as:
$$ a \mathcal{P}(Y = -1|f(X) = p) \leq b \mathcal{P}(Y = 1|f(X) = p) $$
Since $f$ is calibrated, $\mathcal{P}(Y = 1|f(X) = p) = p$, therefore $g^*(p) = 1$ if and only if $p \geq \frac{a}{a+b}$.
\end{proof}

\subsection*{Proof of Theorem~\ref{thm_main}}

\begin{proof}
We will use the following uniform convergence result (Shalev-Shwartz and Ben-David, 2014):

\begin{theorem}
Let $D$ be i.i.d.~samples of $(\mathcal{X} \times \mathcal{Y}, \mathcal{P})$, then with probability at least $1 - \delta$,
\begin{eqnarray}\label{eq_rademacher}
& \sup_{g \in \mathcal{G}} |\frac{1}{n} \sum_{i=1}^n g(x_i, y_i) - \mathbb{E} g(X,Y)| \nonumber \\
\leq & 2 \mathbb{E}_D R_D(\mathcal{G}) + \sqrt{\frac{2 \ln (4 / \delta)}{n}}
\end{eqnarray}
\end{theorem}
In the following we sometimes allow $\mathcal{G}$ to be a collection of functions from $\mathcal{X}$ to $[0, 1]$ in the above results. When used in this sense, we assume that the function will not use $y$ label: $g(x, y) = g(x)$.

Define $\mathcal{F}_{D,p_1,p_2}(f)$ to be the relative frequency of event $\{p_1 < f(x) \leq p_2, y = 1\}$:
$$ \mathcal{F}_{D,p_1,p_2}(f) = \frac{1}{n} \sum_{i=1}^n \mathds{1}_{p_1 < f(x_i) \leq p_2, y_i = 1} $$

Define $\mathcal{F}_{\mathcal{P},p_1,p_2}(f)$ to be the probability of the same event:
$$ \mathcal{F}_{\mathcal{P},p_1,p_2}(f) = \mathcal{P}(p_1 < f(X) \leq p_2, Y = 1) $$

Define $\mathcal{E}_{D,p_1,p_2}(f)$ as the empirical expectation of $f(x) \mathds{1}_{p_1 < f(x) \leq p_2}$:
$$ \mathcal{E}_{D,p_1,p_2}(f) = \frac{1}{n} \sum_{i=1}^n f(x_i) \mathds{1}_{p_1 < f(x_i) \leq p_2} $$

Define $\mathcal{E}_{\mathcal{P},p_1,p_2}(f)$ as the expectation of the same function:
$$ \mathcal{E}_{\mathcal{P},p_1,p_2}(f) = \mathbb{E}[f(X) \mathds{1}_{p_1 < f(X) \leq p_2}] $$

When the context is clear, subscripts $p_1$ and $p_2$ can be dropped. Using these notations, we can rewrite $c(f)$ and $c_{\text{emp}}(f,D)$ as follows: 
$$ c(f) = \sup_{p_1, p_2} |\mathcal{F}_\mathcal{P}(f) - \mathcal{E}_\mathcal{P}(f) | $$
$$ c_{\text{emp}}(f) = \sup_{p_1, p_2} |\mathcal{F}_D(f) - \mathcal{E}_D(f)| $$
	
Note that:
\begin{align*}
& |\sup_{p_1, p_2} |\mathcal{F}_D(f) - \mathcal{E}_D(f) | - \sup_{p_1, p_2} |\mathcal{F}_S(f) - \mathcal{E}_S(f)| | \\
\leq & \sup_{p_1, p_2} | |\mathcal{F}_D(f) - \mathcal{E}_D(f) | - |\mathcal{F}_S(f) - \mathcal{E}_S(f)| | \\  
\leq & \sup_{p_1, p_2} |\mathcal{F}_D(f) - \mathcal{E}_D(f) - \mathcal{F}_S(f) + \mathcal{E}_S(f)|  \\
\leq & \sup_{p_1, p_2} (|\mathcal{F}_D(f) - \mathcal{F}_S(f)| + |\mathcal{E}_D(f) - \mathcal{E}_S(f)|) \\ 
\leq & \sup_{p_1, p_2} |\mathcal{F}_D(f) - \mathcal{F}_S(f)| + \sup_{p_1,p_2} |\mathcal{E}_D(f) - \mathcal{E}_S(f)| 
\end{align*}

Therefore it suffices to show that
\begin{align*}
\mathbf{P} (& \sup_{f,p_1,p_2} |\mathcal{F}_D(f) - \mathcal{F}_S(f)| + \\
& \sup_{f,p_1,p_2} |\mathcal{E}_D(f) - \mathcal{E}_S(f)| > \epsilon) < \delta
\end{align*}

Define
\begin{align*}
\mathcal{H}_1 & = \{\mathds{1}_{p_1 < f(x) \leq p_2, y=1}: p_1, p_2 \in \mathbb{R}, f \in \mathcal{F}\} \\
\mathcal{H}_2 & = \{f(x)\mathds{1}_{p_1 < f(x) \leq p_2}: p_1, p_2 \in \mathbb{R}, f \in \mathcal{F}\} 
\end{align*}

Then we have the following lemma:
\begin{lemma}
Let $\mathcal{H}_1, \mathcal{H}_2$ as defined above, then:
$$ R_D(\mathcal{H}_1) \leq R_D(\mathcal{H}) \quad R_D(\mathcal{H}_2) \leq R_D(\mathcal{H}) $$
\end{lemma}
\begin{proof}
For $R_D(\mathcal{H}_1)$, we have:
\begin{align*}
 & R_D(\mathcal{H}_1) \\
= & \frac{1}{n} \mathbb{E}_{\sigma \sim \{\pm 1\}^n} [ \sup_{p_1, p_2, f} \sum_{i=1}^n \sigma_i \mathds{1}_{p_1 < f(x_i) \leq p_2, y_i=1} ]\\
= & \frac{1}{n} \mathbb{E}_{\sigma \sim \{\pm 1\}^n} [ \sup_{p_1, p_2, f} \sum_{i=1}^n \sigma_i \mathds{1}_{p_1 < f(x_i) \leq p_2} \mathbb{E}_{z_i \in \{\pm 1\}} \max(z_i, y_i) ]\\
\leq & \frac{1}{n} \mathbb{E}_{\sigma, z \sim \{\pm 1\}^n} [\sup_{p_1, p_2, f} \sum_{i=1}^n \mathds{1}_{p_1 < f(x_i) \leq p_2} \sigma_i \max(z_i, y_i)] \\
= & \frac{1}{n} \mathbb{E}_{t \sim \{\pm 1\}^n} [\sup_{p_1, p_2, f} \sum_{i=1}^n t_i \mathds{1}_{p_1 < f(x_i) \leq p_2} ]  \\
= & R_D(\mathcal{H})
\end{align*}
where the last step is because $t_i = \sigma_i \max(z_i, y_i)$ is uniformly distributed over $\{\pm 1\}$ independent of the value of $y_i$. 

For $R_D(\mathcal{H}_2)$, we have:
\begin{align*}
& R_{S_n}(\mathcal{H}_2) \\
= & \frac{1}{n} \mathbb{E}_{\sigma \sim \{\pm 1\}^n} [ \sup_{p_1, p_2, f} \sum_{i=1}^n \sigma_i f(x_i) \mathds{1}_{p_1 < f(x_i) \leq p_2} ]\\
= & \frac{1}{n} \mathbb{E}_{\sigma \sim \{\pm 1\}^n} [ \sup_{p_1, p_2, f} \int_0^1 \sum_{i=1}^n \sigma_i \mathds{1}_{t < f(x_i)} \mathds{1}_{p_1 < f(x_i) \leq p_2} dt ]\\
\leq & \frac{1}{n} \mathbb{E}_{\sigma \sim \{\pm 1\}^n} \int_0^1 [\sup_{p_1, p_2, f} \sum_{i=1}^n \sigma_i \mathds{1}_{\max(p_1,t) < f(x_i) \leq p_2} ] dt \\
= & \frac{1}{n} \mathbb{E}_{\sigma \sim \{\pm 1\}^n} \int_0^1 [\sup_{p_1' \geq t, p_2, f} \sum_{i=1}^n \sigma_i \mathds{1}_{p_1' < f(x_i) \leq p_2} ] dt \\
\leq & \frac{1}{n} \mathbb{E}_{\sigma \sim \{\pm 1\}^n} \int_0^1 [\sup_{p_1', p_2, f} \sum_{i=1}^n \sigma_i \mathds{1}_{p_1' < f(x_i) \leq p_2} ] dt \\
= & \frac{1}{n} \mathbb{E}_{\sigma \sim \{\pm 1\}^n} [\sup_{p_1, p_2, f} \sum_{i=1}^n \sigma_i \mathds{1}_{p_1' < f(x_i) \leq p_2} ] \\
= & R_D(\mathcal{H})
\end{align*}
where the second step is due to $f(x) = \int_0^1 \mathds{1}_{t < f(x)} dt$, and the forth step is just substituting $\max(p_1, t)$ with $p_1'$. Since there is no constraint on $p_1$, the $p_1'$ can take any value greater than or equal to $t$.
\end{proof}

Combining this lemma with the assumptions in the theorem:
\begin{align*}
 \mathbb{E}_D R_D(\mathcal{H}_1) + \sqrt{\frac{2 \ln (8/\delta)}{n}} < \frac{\epsilon}{2} \\
 \mathbb{E}_D R_D(\mathcal{H}_2) + \sqrt{\frac{2 \ln (8/\delta)}{n}} < \frac{\epsilon}{2}
\end{align*}

By Equation~(\ref{eq_rademacher}):
\begin{align*}
\mathbf{P}(\sup_{f, p_1, p_2} |\mathcal{F}_D(f) - \mathcal{F}_S(f)| > \frac{\epsilon}{2}) < \frac{\delta}{2} \\
\mathbf{P}(\sup_{f, p_1, p_2} |\mathcal{E}_D(f) - \mathcal{E}_S(f)| > \frac{\epsilon}{2}) < \frac{\delta}{2}
\end{align*}

\end{proof}

\subsection*{Proof of Claim~\ref{clm_svm}}
\begin{proof}
For any $\sigma \in \{\pm 1\}^n$, we can find a vector $w$ such that for every $X_i$, we have $w^T X_i = \sigma_i$ (this is always possible since the number of equations $n$ is less than the dimensionality $d$). Let $w^* = \frac{Bw}{||w||_2}$ so that $||w^*||_2 = B$, and let $a = \lambda ||w||_2 / B$ and $b = 0$. Then we have:
$$ f(X_i) = \frac{1}{1 + \exp(a (w^*)^T x + b)} = \frac{1}{1 + e^{\lambda \sigma_i}}$$
Let $\lambda \rightarrow -\infty$, then $\sum_{i=1}^n \sigma_i f(X_i) \rightarrow \sum_{i=1}^n \mathds{1}_{\sigma_i = 1}$, and the conclusion of the claim follows easily.
\end{proof}

\begin{algorithm}[t]
	\caption{Isotonic Regression Calibration Algorithm (PAV Algorithm)}\label{alg_calibration}
	\begin{enumerate}
		\item
		For $i = 0, \ldots, n$, Compute $P_i = (i, S_i = \sum_{j \leq i} \mathds{1}_{y_j=1})$
		\item
		Let $cv(P)$ be the convex hull of the set of points $P_i$
		\item
		For $i = 0, \ldots, n$, Let $Z_i = $ intersection of $cv(P)$ and the line $x=i$
		\item
		Compute $z_i = Z_i - Z_{i-1}$
		\item
		Let $g(f_0(x_i)) = z_i$, extrapolate these points to get continuous nondecreasing function $g$.
	\end{enumerate}
\end{algorithm}

\subsection*{The Hypothesis Class $\mathcal{H}$}

In Theorem~\ref{thm_main}, $\mathcal{H}$ is the collection of binary classifiers obtained by thresholding the output of a fuzzy classifier in $\mathcal{F}$. For many hypothesis classes $\mathcal{F}$, the Rademacher Complexity of $\mathcal{H}$ can be naturally bounded. For instance, if $\mathcal{F}$ is the $d$-dimensional generalized linear classifiers with monotone link function, then $\mathbb{E}_{D} R_D(\mathcal{H})$ can be bounded by $O(\sqrt{d\log n/n})$. We remark that $\mathcal{H}$ is different from the hypothesis class $\mathcal{H}_{p_1,p_2}$, where the thresholds are fixed in advance:
$$ \mathcal{H}_{p_1, p_2} = \{ \mathds{1}_{p_1 < f(X) \leq p_2} : f \in \mathcal{H} \} $$
In general, the gap between the Rademacher Complexities of $\mathcal{H}_0$ and $\mathcal{H}_{p_1, p_2}$ can be arbitrarily large. The following example illustrates this point.

\begin{example}
	Let $\mathcal{X} = \{1, \ldots, n\}$, and $A_1, A_2, \ldots, A_{2^n}$ be a sequence of sets containing all subsets of $\mathcal{X}$. Let $\mathcal{H}$ be the following hypothesis space:
	$$ \mathcal{F} = \{ f_i(x) = \frac{i}{2^n} - \frac{1}{2^{n+1}} \mathds{1}_{x \in A_i} : i \in \{1, 2, \ldots, 2^n\} \} $$
	Intuitively, $\mathcal{F}$ contains $2^n$ classifiers, the $i$th classifier produces a output of either $\frac{i}{2^n}$ or $\frac{i}{2^n} - \frac{1}{2^{n+1}}$ depending on whether $x \in A_i$.
	One can easily verify that for any $p_1, p_2$, the VC-dimension (Vapnik and Chervonenkis, 1971) of $\mathcal{H}_{p_1, p_2}$ is at most $2$, but the VC-dimension of $\mathcal{H}$ is $n$.
\end{example}

However, if for any $x \in \mathcal{X}, f \in \mathcal{F}$, we have $f(x) \in P^*$ with $|P^*| < \infty$, then $R_D(\mathcal{H})$ can be bounded using the maximum VC-dimension of $\mathcal{H}_{p_1, p_2}$ and $\log |P^*|$:
\begin{claim}\label{clm_finite_output}
	If for any $f \in \mathcal{F}, x \in \mathcal{X}$, we have $f(x) \in P^*$ where $P^*$ is a finite set, and for all $p_1, p_2 \in \mathbb{R}$, the VC-dimension of hypothesis space $\mathcal{H}_{p_1, p_2}$ is at most $d$, then for any sample $D$ of size $n$ with $n > d + 1$ we have:
	$$ R_D (\mathcal{H}) \leq \sqrt{\frac{2d(\ln \frac{n}{d} + 1) + 4 \ln (|P^*| + 1)}{n}} $$
\end{claim}
\begin{proof} 
	By Massart Lemma (Shalev-Shwartz and Ben-David, 2014), we have: 
	$$ R_D(\mathcal{H}) \leq \sqrt{\frac{2 \ln |\mathcal{H}(D)|}{n}} $$
	where $\mathcal{H}(D)$ is the restriction of $\mathcal{H}$ to $D$. It suffices to show that
	$$ |\mathcal{H}(D)| \leq (|P^*| + 1)^2 (en /d)^d $$
	Note that
	$$ \mathcal{H}(D) = \cup_{p_1, p_2} \mathcal{H}_{p_1, p_2}(D) $$
	
	Since $f(x)$ only takes finite possible values, we only need to consider values of $p_1, p_2$ in $P^* \cup \{-\infty\}$. Therefore by union bound we have
	$$ |\mathcal{H}(S_n)| \leq \sum_{p_1, p_2 \in P^* \cup \{-\infty\}} |\mathcal{H}_{p_1, p_2}(S_n)| $$
	Since each $\mathcal{H}_{p_1, p_2}$ has VC-dimension at most $d$, by Sauer's Lemma (Shalev-Shwartz and Ben-David, 2014): 
	$$\forall p_1, p_2, |\mathcal{H}_{p_1, p_2}(S_n)| \leq (en/d)^d $$
	Combining the last two inequalities, we get the desired result.
\end{proof}

\subsection*{Proof of Claim~\ref{clm_PAV_optimality}}
\begin{proof}
	For reference, the pseudo-code of the PAV algorithm for isotonic regression (Niculescu-Mizil and Caruana, 2005) can be found in Algorithm~\ref{alg_calibration}.
	
	Let $z_i = g(f_0(x_i))$, then we can rewrite the objective function as:
	$$ \max_{a,b} | \sum_{a<i\leq b} (\mathds{1}_{y_i=1} - z_i) | $$
	
	To prove Algorithm~\ref{alg_calibration} also minimizes this objective function, we first state the minimization problem as a linear programming:
	\begin{align*}
	\min \xi_1 + \xi_2 \quad
	\textbf{s.t.  } \quad & \xi_1, \xi_2 \geq 0 \\
	& 0 \leq z_1 \leq z_2 \leq \ldots \leq z_n \leq 1\\
	& \forall 1 \leq k \leq n, \sum_{i \leq k} z_i \geq \sum_{i \leq k} \mathds{1}_{y_i = 1} - n \xi_1 \\
	& \forall 1 \leq k \leq n, \sum_{i \leq k} z_i \leq \sum_{i \leq k} \mathds{1}_{y_i = 1} + n \xi_2
	\end{align*}

	Define $S_k = \sum_{i \leq k} \mathds{1}_{y_i=1}$ and $Z_k = \sum_{i \leq k} z_i$. Then we have the following constraints:
	\begin{align*}
	& \forall 1 \leq k \leq n - 1, Z_k - Z_{k-1} \leq Z_{k+1} - Z_k\\
	& \forall 1 \leq k \leq n, S_k - n \xi_1 \leq Z_k \leq S_k + n \xi_2
	\end{align*}
	
	Let $Z_i^*$ be the solution produced by Algorithm~\ref{alg_calibration}, it should be obvious that $Z_i^* \leq S_i$ for all $i$. Therefore, 
	$$\xi_2^* = \frac{1}{n}\min_i (S_i - Z_i^*) = 0 \quad \xi_1^* = \frac{1}{n}\max_i (S_i - Z_i^*)$$
	We need to prove that $\xi_1^* \leq \xi_1 + \xi_2$ for every feasible solution $(Z_i, \xi_i)$. Suppose $\xi_1^* = \frac{1}{n}(S_k - Z_i^*)$, and $Z_i^*$ lies on the line segment $\{(j, S_j), (k, S_k)\}$. Then we have:
	$$ S_i - n \xi_1^* = Z_i^* = \frac{i-j}{k-j} S_k + \frac{k-i}{k-j} S_j $$
	
	Because of the convexity constraint of $Z$, it must satisfy the following inequality:
	$$ Z_i \leq \frac{i-j}{k-j} Z_k + \frac{k-i}{k-j} Z_j $$
	Computing the difference between these two, we get
	$$ Z_i - S_i + n \xi_1^* \leq \frac{i-j}{k-j} (Z_k - S_k) + \frac{k-i}{k-j} (Z_j - S_j) $$
	Substituting in
	$$ Z_i - S_i \geq - n \xi_1 \quad Z_k - S_k \leq n \xi_2 \quad Z_j - S_j \leq n \xi_2 $$
	We get
	$$n \xi_1^* \leq n \xi_1 + n \xi_2 $$
	which proves the optimality of $Z^*$.
\end{proof}

\subsection*{Properties of Isotonic Regression}

We can prove several interesting properties of isotonic regression using Theorem~\ref{thm_main}.

\begin{claim}\label{clm_PAV_calibrated}
Let $g^*$ be the calibrating function produced by Algorithm~\ref{alg_calibration}, then:
\begin{enumerate}
\item
The empirical calibration measure $c_{\text{emp}}(g^* \circ f_0, D)$ of the calibrated classifier is always $0$.
\item
For any asymmetric loss $(1 - p, p)$ (i.e., each false negative incurs $1 - p$ cost and each false positive incurs $p$ cost), the empirical loss of the calibrated classifier is always no greater than the original classifier (both using the optimal decision threshold $p$):
\begin{align*}
& \sum_{i=1}^n [(1 - p) \mathds{1}_{g^*(f_0(x_i)) \leq p, y_i = 1} + p \mathds{1}_{g^*(f_0(x_i)) > p, y_i = 0}] \\
\leq & \sum_{i=1}^n [ (1 - p) \mathds{1}_{f_0(x_i) \leq p, y_i = 1} + p \mathds{1}_{f_0(x_i) > p, y_i = 0}]
\end{align*}
In particular, when $p = 0.5$, the empirical accuracy of the calibrated classifier is always greater than or equal to the empirical accuracy of the original classifier.
\end{enumerate}
\end{claim}
\begin{proof}
Throughout the proof, let $C$ be the convex hull computed in Algorithm~\ref{alg_calibration}:
$$C = \{(i_0 = 0, 0), (i_1, S_{i_1}), \ldots, (i_{m-1}, S_{i_{m-1}}), (i_m = n, S_n)\}$$ 
We will use the following notations:
$$z_i = g^*(f_0(x_i)) \quad Z_k = \sum_{i=1}^k z_i \quad S_k = \sum_{i=1}^k \mathds{1}_{y_i=1}$$
\begin{enumerate}
\item
For any $p_1, p_2$, let $l, r$ be such that:
$$ l = \max_{k \leq n, z_k \leq p_1} k \qquad r = \max_{k \leq n, z_k \leq p_2} k $$
If no such $k$ exists, let $l, r$ be $0$ respectively. By Algorithm~\ref{alg_calibration}, we have
$$ \forall i_j < k \leq i_{j+1}, z_k = \frac{S_{i_{j+1}} - S_{i_j}}{i_{j+1}- i_j} $$
Thus we have $(l, S_l), (r, S_r) \in C$, $Z_l = S_l, Z_r = S_r$, and therefore
\begin{align*}
& \sum_{i=1}^n \mathds{1}_{p_1 < z_i \leq p_2, y_i = 1} - \sum_{i=1}^n \mathds{1}_{p_1 < z_i \leq p_2} z_i \\
 = & (Z_r - Z_l) - (S_r - S_l) = 0
\end{align*}
which implies that $c_{emp}(g^* \circ f_0) = 0$
\item
Let $a = \max\{i:f_0(x_i) \leq p\}, b = \max\{i: z_i \leq p\}$, then we need to show that
\begin{align*}
(1 - p)\sum_{i=1}^b \mathds{1}_{y_i = 1} + p \sum_{i=b+1}^n \mathds{1}_{y_i=0} \\ \leq (1 - p)\sum_{i=1}^a \mathds{1}_{y_i = 1} + p \sum_{i=a+1}^n \mathds{1}_{y_i=0}
\end{align*}
We consider two separate cases:
\begin{enumerate}
\item
$a \leq b$, in this case we only need to show that
$$ \sum_{i=a+1}^b [p \mathds{1}_{y_i=0} - (1-p) \mathds{1}_{y_i=1}] \geq 0 $$
or equivalently,
$$ p[(b - a) - (S_b - S_a)] - (1-p)(S_b - S_a) \geq 0 $$
Rearrange terms, it suffices to show
$$ p(b-a) - (S_b - S_a) \geq 0$$
Since $S_b = Z_b, S_a \geq Z_a$
$$ S_b - S_a \leq Z_b - Z_a \leq z_b (b - a) \leq p (b - a) $$
\item
$a > b$, in this case we only need to show
$$ \sum_{i=b+1}^a [p \mathds{1}_{y_i=0} - (1-p) \mathds{1}_{y_i=1}] \leq 0 $$
or equivalently,
$$ p[(a - b) - (S_a - S_b)] - (1-p)(S_a - S_b) \leq 0 $$
Rearrange terms, it suffices to show
$$ p(a-b) - (S_a - S_b) \leq 0$$
Since $S_b = Z_b, S_a \geq Z_a$
$$ S_a - S_b \geq Z_a - Z_b \geq z_{b+1} (a - b) \geq p (a - b) $$
\end{enumerate}
\end{enumerate}
\end{proof}

We can also use Theorem~\ref{thm_main} to derive the following non-asymptotic convergence result of Algorithm~\ref{alg_calibration}.
\begin{claim}\label{clm_PAV_converge}
Let $F(t) = \mathcal{P}(f_0(X) \leq t)$ be the distribution function of $f_0(X)$, and define $G(t)$ as:
$$ G(t) = \mathcal{P}(f_0(X) \leq t, Y = 1) $$
Let $cv : [0,1] \rightarrow [0,1]$ be the convex hull of all points $(F(t), G(t))$ for all $t \in [0,1]$. Define $G_e$ as:
$$ G_e(t) = \mathbb{E}[\mathds{1}_{f_0(X) \leq t} g^*(f_0(X))] $$
Then under the same condition in Theorem~\ref{thm_main},
$$ \mathbf{P}(\sup_t |G_e(t) - cv(F(t))| > 2\epsilon) < 5 \delta $$
In particular, if $\mathcal{P}(Y=1|f_0(X))$ is monotonically increasing, then
$$ \mathbf{P}(\sup_t |G_e(t) - G(t)| > 2\epsilon) < 5 \delta $$
\end{claim}

Let us explain the intuition behind this claim: $F(t)$ is the percentage of data points satisfying $f_0(X) \leq t$, and $G(t)$ is $F(t)$ times the conditional probability of $Y = 1$ in the region $\{f_0(X) \leq t\}$. Now consider points $P_i = (i, S_i)$ in Algorithm~\ref{alg_calibration}, it is not hard to show that as $n \rightarrow \infty$, the limit of points $P_i$ are the curve $(F(t), G(t)), t \in [0,1]$ (after proper scaling). Similarly, $G_e(t)$ is $F(t)$ times the expected value of $g^*(f_0(X))$ in the region $\{f_0(X) \leq t\}$, and it is not hard to show that $(F(t),G_e(t))$ is the limit of $(i,Z_i)$ (after proper scaling). Now the claim states that in the PAV algorithm, $(F(t),G_e(t))$ converge uniformly to the convex hull of $(F(t), G(t))$, which should not be surprising, since we explicitly computed the convex hull of $\{P_i\}$ in Algorithm~\ref{alg_calibration}. 

When $\mathcal{P}(Y=1|f_0(X))$ is monotonically increasing w.r.t. $f_0(X)$, $(F(t), G(t))$ is convex, and Claim~\ref{clm_PAV_converge} immediately implies that $G_e(t)$ will converge uniformly to $G(t)$. In this case, the PAV algorithm will eventually recover the ``true'' link function $g^*(f_0(X)) = \mathcal{P}(Y=1|f_0(X))$ given sufficient training samples, and Claim~\ref{clm_PAV_converge} provides a rough estimate of the number of samples required to achieve the desired precision.

\begin{proof}
Throughout the proof, let $C$ be the convex hull computed in Algorithm~\ref{alg_calibration}:
$$C = \{(i_0 = 0, 0), (i_1, S_{i_1}), \ldots, (i_{m-1}, S_{i_{m-1}}), (i_m = n, S_n)\}$$ 
We will use the following notations:
$$z_i = g^*(f_0(x_i)) \quad Z_k = \sum_{i=1}^k z_i \quad S_k = \sum_{i=1}^k \mathds{1}_{y_i=1}$$

We will use the following facts in the proof of Theorem~\ref{thm_main}:
$$ \mathbf{P}(\sup_{g, p_1, p_2} | \mathcal{F}_D(g \circ f_0) - \mathcal{F}_\mathcal{P}(g \circ f_0) | > \frac{\epsilon}{2}) < \frac{\delta}{2} $$
$$ \mathbf{P}(\sup_{g, p_1, p_2} | \mathcal{E}_D(g \circ f_0) - \mathcal{E}_\mathcal{P}(g \circ f_0) | > \frac{\epsilon}{2}) < \frac{\delta}{2} $$

For any $t \in [0,1]$, let $g'$ be any continuous increasing function from $[0,1]$ to $[0,1]$. Let $k = \max\{i: f_0(x_i) \leq t\}, p_1 = -\infty, p_2 = g'(t)$ in the above inequalities, then we have:
\begin{equation}\label{eq_1}
\mathbf{P}(|\frac{1}{n} S_k - G(t)| > \frac{\epsilon}{2}) < \frac{\delta}{2}
\end{equation}
$$ \mathbf{P}(|\frac{1}{n} \sum_{i=1}^k g'(f_0(x_i)) - \mathbb{E}[\mathds{1}_{f_0(X) \leq t} g'(f_0(X))]| > \frac{\epsilon}{2}) < \frac{\delta}{2} $$
Let $g'$ be such that $||g' - g^*||_\infty < \lambda$, where $\lambda > 0$ can be arbitrarily small. Let $\lambda \downarrow 0$, then the second inequality implies
\begin{equation}\label{eq_2}
\mathbf{P}(|\frac{1}{n} Z_k - G_e(t)| > \frac{\epsilon}{2}) < \frac{\delta}{2}
\end{equation}
Let $g'$ be such that $|g'(x) - 1| < \lambda$ for any $x$. Let $\lambda \downarrow 0$, then the second inequality implies
\begin{equation}\label{eq_3}
\mathbf{P}(|\frac{1}{n}k - F(t)| > \frac{\epsilon}{2}) < \frac{\delta}{2}
\end{equation}

For any $t \in [0, 1]$, let $k = \max\{i: f_0(x_i) \leq t\}$.
Let $[i_{j-1} = l,i_j = r]$ be the segment of $C$ with $l < k \leq r$. Then we have
$$ z_{l+1} = \ldots = z_k = \ldots = z_r$$
$$ S_l = Z_l = Z_k - (k - l) z_k $$
$$ S_r = Z_r = Z_k + (r - k) z_k $$
On the other hand, by (\ref{eq_1}), with probability at least $1 - \delta$:
$$ \frac{1}{n} S_l \geq G(f_0(x_l)) - \frac{\epsilon}{2} \quad \frac{1}{n} S_r \geq G(f_0(x_r)) - \frac{\epsilon}{2} $$
Since $cv$ is the convex hull of $(F(t),G(t))$, we have
$$ q G(f_0(x_l)) + (1 - q) G(f_0(x_r)) \geq cv(F(t)) $$
where $q = \frac{F(f_0(x_r)) - F(t)}{F(f_0(x_r)) - F(f_0(x_l))}$. Combining all, with probability at least $1 - \delta$:
$$ \frac{1}{n} Z_k + \frac{1}{n} [ql + (1 - q)r - k] z_k + \frac{\epsilon}{2} \geq cv(F(t)) $$
By (\ref{eq_3}), with probability at least $1 - \frac{3}{2}\delta$:
$$ \frac{1}{n}l \leq F(f_0(x_l)) + \frac{\epsilon}{2} \quad \frac{1}{n}r \leq F(f_0(x_r)) + \frac{\epsilon}{2} $$
$$ \frac{1}{n}k \geq F(t) - \frac{\epsilon}{2} $$
Therefore, we have with probability at least $1 - \frac{5}{2} \delta$,
$$ \frac{1}{n} Z_k + \frac{3\epsilon}{2} \geq cv(F(t)) $$
Then by (\ref{eq_2}), with probability at least $1 - 3 \delta$,
$$ G_e(t) + 2 \epsilon \geq cv(F(t)) $$

Conversely, suppose $(F(t),cv(F(t)))$ is on the line segment between $(F(a),G(a))$ and $(F(b),G(b))$, then
$$ G(a) = cv(F(t)) - w(F(t) - F(a)) $$
$$ G(b) = cv(F(t)) + w(F(b) - F(t)) $$
where $w = \frac{G(b) - G(a)}{F(b) - F(a)}$ (if $F(a) = F(b)$ then just let $w = 1$).

By (\ref{eq_1}) and (\ref{eq_2}) and the fact that $S_k \geq Z_k$, with probability at least $1 - 2 \delta$:
$$ G(a) + \epsilon \geq G_e(a) \quad G(b) + \epsilon \geq G_e(b) $$
Also since $(F(t),G_e(t))$ is convex, we have:
$$ q G_e(a) + (1 - q) G_e(b) \geq G_e(t) $$
where $q = \frac{F(b) - F(t)}{F(b) - F(a)}$. Combining all above, with probability at least $1 - 2 \delta$:
$$ cv(F(t)) + \epsilon \geq G_e(t) $$
Combining two directions, the proof is complete.
\end{proof}

\subsection*{Discussion on Kakade's Algorithm (2011)}

Kakade's algorithm minimizes the following squared loss objective function:
$$ \mathcal{L}(u, w) = \sum_{i=1}^n (y_i - u(w \cdot x_i)) $$
where $u$ is a non-decreasing $1$-Lipschitz function and $w$ satisfies $||w|| \leq W$. In each iteration, the algorithm first fix $u$ and search for the optimal $w$ that minimizes the squared loss, then fix $w$ and run a slightly modified version of the PAV algorithm (Algorithm~\ref{alg_calibration}) to find the optimal $u$. 

In Claim~\ref{clm_PAV_calibrated}, we proved that the PAV algorithm always produce a calibrated classifier, therefore Kakade's algorithm can be viewed as alternating between the following two steps:
\begin{enumerate}
\item
Search for the parameter $w$ that minimizes the squared loss $\mathcal{L}(u, w)$.
\item
Find the link function $u$ such that $u(w \cdot x)$ is empirically calibrated.
\end{enumerate}

In other words, each iteration of Kakade's algorithm can be viewed as first optimizing the objective function $\mathcal{L}(u, w)$, then projecting $u(w \cdot x)$ onto the space of empirically calibrated classifiers. An interesting question here is whether the algorithm would still work if we replace the squared loss function with any other loss function in the first step. 
}

\end{document}